\def\year{2020}\relax
\documentclass[letterpaper]{article} 
\usepackage{aaai20}  
\usepackage{times}  
\usepackage{helvet}  
\usepackage{courier}  
\usepackage{url}  
\usepackage{graphicx}  
\usepackage{amsmath}
\usepackage{amsfonts}
\usepackage{amssymb}
\usepackage{graphicx}
\usepackage{multirow}
\usepackage{amsthm}
\usepackage{algorithm}
\usepackage{algpseudocode}
\newtheorem{theorem}{Theorem}
\newtheorem{lemma}{Lemma}
\newtheorem{corollary}{Corollary}
\newtheorem{claim}{Claim}
\nocopyright

\begin{document}

\title{Achieving Differential Privacy in Vertically Partitioned Multiparty Learning}

\author{
Depeng Xu, Shuhan Yuan, Xintao Wu\\
University of Arkansas\\
\{depengxu,sy005,xintaowu\}@uark.edu
}

\maketitle
\begin{abstract}
Preserving differential privacy has been well studied under centralized setting. However, it's very challenging to preserve differential privacy under multiparty setting, especially for the vertically partitioned case. In this work, we propose a new framework for differential privacy preserving multiparty learning in the vertically partitioned setting. Our core idea is based on the functional mechanism that achieves differential privacy of the released model by adding noise to the objective function. We show the server  can simply dissect the objective function into single-party and cross-party sub-functions, and allocate computation and perturbation of their polynomial coefficients to local parties. Our method needs only one round of noise addition and secure aggregation. The released model in our framework achieves the same utility as applying the functional mechanism in the centralized setting. Evaluation   on real-world and synthetic datasets for linear and logistic regressions  shows the effectiveness of our proposed method.

\end{abstract}

\section{Introduction}
Rapid growth of model technology is largely driven by data. In most industries, data exist in the form of isolated islands. Federated learning is proposed to build machine learning models based on distributed datasets across multiple parties \cite{McMahanMRHA17,YangLCT19}. In particular, vertically partitioned multiparty learning is applicable when parties share the same record ID space but differ in feature space, such as using user experience on the web to support decisions on  healthcare. Understandably, parties do not want to share raw data or statistics due to privacy concerns. How to build a global model through data barrier while preserving local parties' privacy  is a challenging problem. 

Differential privacy is a standard privacy preserving scheme to achieve opt-out right of individuals \cite{Dwork2006}. In general, differential privacy guarantees the query results or the released model cannot be exploited by attackers to derive whether one particular record is present or absent in the underlying dataset. Many mechanisms have been proposed to achieve differential privacy \cite{Dwork:2011,4389483,Chaudhuri2011,Nissim:2007}.
For example, the classic Laplace mechanism  injects random noise into the released results such that the inclusion or exclusion of a single   record   makes no statistical difference   \cite{Dwork2006}. For machine learning models,
\cite{SongCS13,AbadiCGMMT016} develop methods of adding noise to gradients to preserve differential privacy of training data. Functional mechanism \cite{zhang2012functional}, which adds noise to the objective function rather than parameters of built models, has also been shown great success in deep learning models \cite{DBLP:conf/aaai/Phan0WD16}. 

Recently, several works propose  to train privacy preserving models under decentralized settings. Research in \cite{Shokri2015} proposes a collaborative deep learning framework in which participants train independently and share only subsets of updates of parameters under the horizontally distributed setting. However, it is not applicable in the vertically partitioned setting. This is because we cannot partition the gradients based on features and  each local  party  needs to collect raw data of those features owned by other parties, which requires   extensive use of secure multiparty computation to update gradients in each iteration.  

There have been several research works on building privacy preserving models in the vertically partitioned setting. Research in \cite{Heinze-DemlMM17} develops a framework for private data sharing for the purpose of statistical estimation. Each party  communicates perturbed random projections of their locally held features to ensure differential privacy. However, the task focuses on the statistical estimation of coefficients rather than releasing a jointly trained model in our context.  Research in \cite{LouC18} develops a distributed private block-coordinate Frank-Wolfe algorithm under arbitrary sampling. They design an active feature sharing scheme by utilizing private Johnson-Lindenstrauss transform to update local partial gradients in a differentially private and communication efficient manner. However, the gradient perturbation requires noise addition in each iteration, which is difficult to achieve good utility-privacy tradeoff, as shown in our evaluation.  In ensemble learning, research in \cite{YaoGKTCD019} proposes to enhance privacy preserving logistic regression by feature-wise partitioned stacking. The proposed method is combined with hypothesis transfer learning to enable learning across different organizations.  However, this research does not really apply to vertical partitioned learning as the high-level model still needs to access all data to construct meta-data set when training private logistic regression.

In this work, we propose a new framework for differential privacy preserving multiparty learning in the vertically partitioned setting. Our core idea is based on the functional mechanism that achieves differential privacy of the released model by adding noise to the objective function. In the framework, we show the server can simply dissect the objective function into single-party and cross-party sub-functions and rewrite them in the polynomial form. For the coefficients in the polynomial form related to one single party, they can be calculated by each party in a differentially private manner. For those coefficients related to two or multiple parties, we apply secure vector multiplication and then add noise before sending to server. The server then solves the perturbed objective function in the server side and releases the private model. Our method needs only one round of noise addition and secure aggregation. Hence, both good privacy-utility tradeoff and computational efficiency can be achieved.  In fact, the released model in our framework achieves the same utility as applying the functional mechanism in the centralized setting. We evaluate our method on real-world and synthetic datasets for linear and logistic regressions. The experiment results show the effectiveness of our proposed method. 

\section{Preliminaries}
In this section, we revisit how to achieve differential privacy in the centralized setting. 
Consider a dataset $D$ with $n$ users. Each user's information is a record $t_i=\{\mathbf{x}_i,y_i\}$, where $\mathbf{x}_i$ is the user's feature information and $y_i$ is the user's label. The total number of features  is $d$.  We assume that $x_{ia}\in[-1,1]$ for $a\in[1,d]$ and $y\in[-1,1]$ for linear regression or $y\in\{0,1\}$ for logistic regression. 
The objective is to build a model $\hat{y} = q(\mathbf{x};\mathbf{w})$  from   $D$ that achieves differential privacy. To fit $\mathbf{w}$, we have an objective function $f_D(\mathbf{w})=\sum_{i=1}^n f(t_i;\mathbf{w})$ that takes $t_i$ and $\mathbf{w}$ as input. The optimal model parameter   is defined as: ${\mathbf{w}} = \arg\min\limits_\mathbf{w} \sum_{i=1}^n f(t_i;\mathbf{w})$. 
We use linear regression and logistic regression as  examples in this paper.

\subsection{Differential Privacy}
Differential privacy guarantees   output of a query $q$ be insensitive to the presence or absence of   one   record in a dataset. 

	\textit{Differential privacy} \cite{Dwork2006}.  A mechanism $\mathcal{M}$ satisfies $\varepsilon$-differential privacy, if for all neighboring datasets $D$ and $D'$ that differ in exactly one record and all subsets $Z$ of $\mathcal{M}$'s range:
	\begin{equation}\Pr(\mathcal{M}(D)\in Z) \leq \exp(\varepsilon)\cdot \Pr(\mathcal{M}(D')\in Z).\nonumber\end{equation}
The parameter $\varepsilon$ denotes the privacy budget (smaller values indicate  stronger privacy guarantee). 

	\textit{Global sensitivity} \cite{Dwork2006}. Given a query $q$: $D \rightarrow \mathbb{R}^d$, the global sensitivity $\Delta$ is defined as $\Delta=\max_{D,D'} ||q(D)-q(D')||_1$.
The global sensitivity measures the maximum possible change in $q(D)$ when one record in the dataset changes.
The Laplace mechanism is a popular method to achieve differential privacy. It adds identical independent noise into each output value of $q(D)$. 

	\textit{Laplace mechanism} \cite{Dwork2006}. Given a dataset $D$ and a query $q$, a mechanism $\mathcal{M}(D)=q(D)+\boldsymbol{\eta}$ satisfies $\varepsilon$-differential privacy, where $\boldsymbol{\eta}$ is a random vector drawn from $Lap(\frac{\Delta}{\varepsilon})$ \footnote{The Laplace distribution $Lap(\boldsymbol{\eta}|\mu,\sigma)$
		with mean $\mu$ and scale $\sigma $ has probability density function $Lap({\eta}|\mu,\sigma)= \frac{1}{2\sigma}\exp(\frac{|x-\mu|}{\sigma})$. Its variance is $2\sigma^2$. Note $\mu=0$ if not specified.}.

Alternately, adding Gaussian noise $N(0,\sigma^2)$ with   $\sigma$ calibrated to $\Delta\ln{(1/\delta)}/\varepsilon$, one can achieve $(\varepsilon,\delta)$-differential privacy, where $\delta>0$ gives relaxed differential privacy.

\subsection{Functional Mechanism}
Functional mechanism \cite{zhang2012functional} is a differentially private method designed for optimization based models. It achieves $\varepsilon$-differential privacy by injecting noise into the objective function   and returns privacy preserving parameter $\bar{\mathbf{w}}$ that minimizes the perturbed objective function.

Because the objective function $f_D(\mathbf{w})$ is a complicated function of $\mathbf{w}$, the functional mechanism exploits the polynomial representation of $f_D(\mathbf{w})$. The model parameter $\mathbf{w}$ is a vector that contains $d$ values $w_1,w_2,\cdots,w_d$. Let $\phi(\mathbf{w})$ denote a product of $w_1,w_2,\cdots,w_d$, i.e., $\phi(\mathbf{w})=w_1^{c_1} \cdot w_2^{c_2} \cdots w_d^{c_d}$ for some $c_1,c_2,\cdots, c_d \in \mathbb{N}$. Let $\mathbf{\Phi}_j$ ($j\in \mathbb{N}$) denote the set of all products of $w_1,w_2,\cdots,w_d$ with degree $j$, i.e., $\mathbf{\Phi}_j = \{w_1^{c_1}  w_2^{c_2} \cdots w_d^{c_d}|\sum_{l=1}^d c_l = j\}$. For example, $\mathbf{\Phi}_1=\{w_1,w_2,\cdots,w_d\}$, and $\mathbf{\Phi}_2=\{w_a \cdot w_b|a,b \in [1,d]\}$.

Based on the Stone-Weierstrass Theorem \cite{rudin1953principles}, any continuous and differentiable function can be expressed in the polynomial representation. Hence, the objective function $f_D(\mathbf{w})$ can be expressed as a polynomial of $w_1,w_2,\cdots,w_d$, for some $J \in \mathbb{N}$: 
\begin{equation}
f_D(\mathbf{w}) = \sum_{i=1}^n \sum_{j=0}^J \sum_{\phi \in \mathbf{\Phi}_j} \lambda_{\phi t_i} \phi(\mathbf{w}),
\label{eq:polyfun}
\end{equation}
where $\lambda_{\phi t_i} \in \mathbb{R}$ denotes the coefficient of $\phi(\mathbf{w})$.

Functional mechanism perturbs the objective function $f_D(\mathbf{w})$ by injecting Laplace noise  into its polynomial coefficients $\bar{\lambda}_\phi=\sum_{i=1}^n\lambda_{\phi t_i}+Lap(\frac{\Delta_f}{\varepsilon})$, where the global sensitivity  of ${f}_D(\mathbf{w})$ is $\Delta_f=2\max\limits_t \sum_{j=1}^J \sum_{\phi \in \mathbf{\Phi}_j} ||\lambda_{\phi t}||_1$. Then the model parameter $\bar{\mathbf{w}}$ is derived by minimizing the perturbed function $\bar{f}_D (\mathbf{w})$.

\subsubsection{Application to linear regression.}
A linear regression on $D$ returns a prediction function $\hat{y}_i = q(\mathbf{x}_i; \mathbf{w}) = \mathbf{x}_i^T \mathbf{w}$.
The objective function of linear regression is defined as:
\begin{equation}
\label{eq:linfun}
\begin{aligned}
    &f_D(\mathbf{w}) = \sum_{i=1}^n(y_i-\mathbf{x}_i^T \mathbf{w})^2 = \sum_{i=1}^n(y_i)^2\\
    &-\sum_{a=1}^d(2\sum_{i=1}^n y_i{x}_{ia}) {w}_a 
    +\sum_{1\leq a,b \leq d}(\sum_{i=1}^n{x}_{ia}{x}_{ib}){w}_a\cdot{w}_b.
\end{aligned}
\end{equation}
We get the polynomial coefficients  
$ \lambda_{\phi_0} = \sum\limits_{i=1}^n(y_i)^2$, $ \lambda_{ w_a} = -2\sum\limits_{i=1}^n y_i{x}_{ia}$, and $ \lambda_{ {w}_a\cdot{w}_b} = \sum\limits_{1\leq a,b \leq d}\sum\limits_{i=1}^n{x}_{ia}{x}_{ib}$.
and then add $Lap(\frac{\Delta_f}{\varepsilon})$ to the coefficients, where the global sensitivity  of ${f}_D(\mathbf{w})$ for linear regression is $\Delta_f=2(1+2d+d^2)$.

\subsubsection{Application to logistic regression.}
A logistic regression on $D$ returns a function which predicts $\hat{y}_i = 1$ with probability $\hat{y}_i = q(\mathbf{x}_i; \mathbf{w}) = \exp (\mathbf{x}_i^T \mathbf{w})/(1+\exp(\mathbf{x}_i^T \mathbf{w}))$.
The objective function of logistic regression is defined as:
\begin{equation}
\label{logfun}
f_D(\mathbf{w}) = \sum_{i=1}^n\left[\log(1+\exp(\mathbf{x}_i^T\mathbf{w}))-y_i\mathbf{x}_i^T\mathbf{w}\right].
\end{equation}
As the polynomial form of  Equation \ref{logfun} contains terms with unbounded degrees, to apply the functional mechanism, it is rewritten as the approximate polynomial representation based on Taylor expansion \cite{zhang2012functional}:
\begin{equation}
{f}_D(\mathbf{w}) = \Big(\sum\limits_{i=1}^n\sum\limits_{j=0}^2\frac{f_1^{(j)}(0)}{j!}\left(\mathbf{x}_i^T\mathbf{w}\right)^j\Big)-\Big(\sum\limits_{i=1}^n y_i\mathbf{x}_i^T\Big)\mathbf{w},
\label{eq:polylog}
\end{equation}
where $f_1(\cdot)=\log(1+\exp(\cdot))$, $J=2$.
We get the  polynomial coefficients $\lambda_{w_a} = \sum\limits_{i=1}^n 
\big(\tfrac{f_1^{(1)}(0)}{1!}-y_i\big){x}_{ia}$ and 
$\lambda_{ {w}_a\cdot{w}_b} = \sum\limits_{1\leq a,b \leq d}\sum\limits_{i=1}^n \tfrac{f_1^{(2)}(0)}{2!}{x}_{ia}{x}_{ib}$, and then
add $Lap(\frac{\Delta_f}{\varepsilon})$ to the coefficients, where the global sensitivity  of ${f}_D(\mathbf{w})$ for logistic regression is  $\Delta_f=\frac{d^2}{4}+d$.

\subsection{Secure Share of Scalar Product}
For privacy concerns,   actual data shall be protected and cannot be known to each party or the server. Research in  \cite{BonehGN05} proposes BGN ``doubly homomorphic'' encryption algorithm which simultaneously supports one multiplication and unlimited number of addition operations, i.e.  BGN enables two parties  to compute the scalar product $\sum\limits_{i=1}^n x_{ia}\cdot x_{ib}$ given the ciphertexts of two vectors  $\{x_{1a},x_{2a},\ldots,x_{na}\}$ and $\{x_{1b},x_{2b},\ldots,x_{nb}\}$.  
Research in \cite{YuanY14}  modifies  BGN algorithm to split the decryption capability among multiple participants for collusion-resistance decryption. 
Each participant first encrypts its private data and then uploads the ciphertexts to the server. The server then executes the operations over the ciphertexts and returns the encrypted results to the participants. Each pair of participants jointly  decrypts the actual result. During this process,  server learns no private data of a participant even if they collude with all the rest participants. Through offloading the computation tasks to the resource-abundant cloud server, this scheme makes the computation and communication complexity on each participant independent to the number of participants.

\section{Achieving Differential Privacy in Vertically Partitioned Multiparty Learning}

In this section, we propose a framework of achieving differential privacy in vertically partitioned multiparty learning based on functional mechanism.

\subsection{Problem Statement}

In the vertically partitioned multiparty setting, each user's information is held by $K$  parties separately. 
Each party $P_k$ owns a disjoint dataset $D^{k}$ on feature set $\mathbf{X}^{k}$, where $|\mathbf{X}^{k}|=d^{k}$. Similarly, $\mathbf{w}^{k}$ denotes subset of $\mathbf{w}$  corresponding to $\mathbf{X}^{k}$. Label $Y$ is not shared by all parties. Without loss of generality, we simply assume party $P_1$ holds the label.

A server coordinates $K$ parties to build a multiparty learning model.  The server is honest but curious. 
It aims to release a model trained from the whole dataset $D$ and to ensure the released model satisfies $\varepsilon$-differential privacy regarding to  $D$.
The parties  provide necessary information to the server and help server to build the $\varepsilon$-differentially private global model.
But they
do not trust the server or each other in terms of sharing users' private information from their   local datasets. Each party can share statistics  in a differentially private manner. If a computation involves at least two parties, it is conducted by a secure multiparty computation.  For  party $P_1$, it shares the label with other parties upon request through secure multiparty computation.
On top of that, each party $P_k$ cares about the level of differential privacy achieved regarding to its  sub-dataset $D^{k}$.
In the training process,   the local party $P_k$  achieves  
$\varepsilon^{(k)}$-differential privacy, where $\varepsilon^{(k)}$ is ideally a smaller  privacy level than $\varepsilon$.  

The goal is  to reduce the amount of secure multiparty computation and  noise addition to the minimum   while keeping local information secure and private.

\begin{figure}
	\centering
	\includegraphics[width=0.9\linewidth]{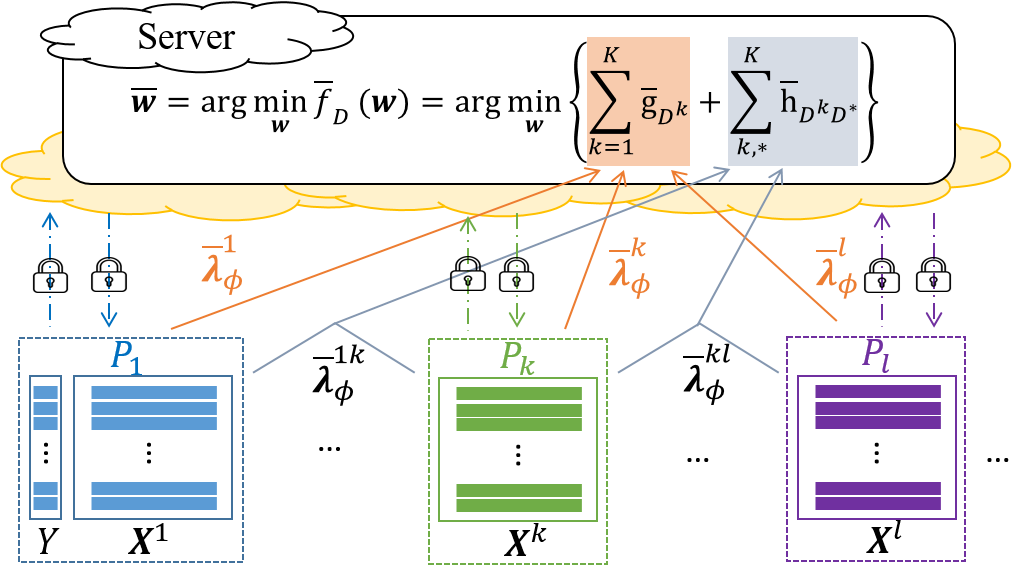}
	\caption{The framework of achieving differential privacy in vertically partitioned multiparty learning based on functional mechanism. (1) Dissect $f$ into sum of  $g$ and  $h$; (2) Collect  $\bar{\boldsymbol{\lambda}}_{\phi}^k$ from each party $P_k$; (3) Secure vector multiplication and collect $\bar{\boldsymbol{\lambda}}_{\phi}^{k*}$; (4) Solve $\bar f$.}
	\label{fig:VFM}
\end{figure}

\subsection{General Model Framework}

We apply functional mechanism in vertically partitioned multiparty learning. Functional mechanism does not inject noise directly into the regression results, but ensures privacy by perturbing   objective function of the regression analysis.  The server only collects information of the objective function at the beginning. The objective function can be dissected based on features, so computation and perturbation of the coefficients can be allocated to local parties by local feature sets. For some operations involving multiple parties, the server conducts secure multiparty computation with the parties.
Once the perturbed coefficients are collected from local parties, the server solves the perturbed objective function and releases the private model. 
Figure \ref{fig:VFM} illustrates our proposed framework of functional mechanism in vertically partitioned multiparty learning. 

The  procedure of  the framework   is shown as Algorithm \ref{FuncMechs}.Overall, there are four steps: 
\\(1) Server dissects objective function $f$ into sum of single-party sub-function $g$ and cross-party sub-function $h$, set and allocates the corresponding coefficients $ \{{\boldsymbol\lambda}_{\phi}^k\}^K,\{{\boldsymbol\lambda}_{\phi}^{k*}\}^K$ for each party $P_k$ to compute (Line \ref{line:f}-\ref{line:f2}). 
We use $\boldsymbol{\lambda}_{\phi }^{k}, \boldsymbol{\lambda}_{\phi }^{k*}$ to denote all the single-party coefficients from party $P_k$ and all the cross-party coefficients involving party $P_k$, respectively.
The server calculates the scale of noise needed to achieve $\varepsilon$-differential privacy and informs the parties (Line \ref{line:delta}).
\\ (2) Each party $P_k$ computes polynomial coefficients $\boldsymbol\lambda_{\phi }^{k}$ in single-party sub-function $g$ from $D^k$  and sends noisy single-party  coefficients $\bar{\boldsymbol\lambda}_{\phi}^k$   to server (Line \ref{line:ks}-\ref{line:ksn}). 
\\(3) For polynomial coefficients $\boldsymbol\lambda_{\phi }^{k*}$ in cross-party sub-function $h$, it involves data from   $P_k$ and $P_*$, such as  $\sum\limits_{i=1}^n y_i{x}_{ia}^k$, where $y_i$ is from party $P_1$, and $\sum\limits_{1\leq a,b \leq d}\sum\limits_{i=1}^n{x}_{ia}^k{x}_{ib}^l$, where $X_a \in \mathbf{X}^k,X_b \in \mathbf{X}^l$. 
Note that $\lambda_{\phi }^{k*}$ is a scalar product of two vectors from party $P_k$ and $P_*$.
All parties send encrypted vectors  of user information to server and receive back the securely aggregated polynomial coefficients ${\boldsymbol\lambda}_{\phi}^{k*}$ following the secure vector multiplication method by \cite{YuanY14} (Line \ref{line:smc}).
The parties add Laplace noise to the results and send   $\bar{\boldsymbol\lambda}_{\phi}^{k*}$ back to server (Line \ref{line:h}-\ref{line:h2}).
\\(4) Server receives all $\bar\lambda_{\phi}$, solves  noisy objective function $\bar f$ and releases the differentially private model (Line \ref{line:fgs}-\ref{line:return}).

\begin{algorithm}[htb]
	\begin{algorithmic}[1]
	    \State Set  ${f}_D(\mathbf{w}) $ by Equation \ref{eq:polyfuns}. \Comment{Server}
	    \label{line:f}
	    \State Allocate $\{{\boldsymbol\lambda}^k\}^K,\{{\boldsymbol\lambda}^{k*}\}^K$ to parties  \Comment{Server}
	    \label{line:f2}
	    \State Set  $\Delta_f=2\max\limits_t \sum_{j=1}^{J}\sum_{\phi \in \mathbf{\Phi}_j}||\lambda_{\phi t}||_1 $  \Comment{Server}
	    \label{line:delta}
	    \For {each party $P_k $ }
	    \For{each $\lambda_{\phi }^{k}\in\boldsymbol\lambda_{\phi }^{k}$}
	    \State Compute $\lambda_{\phi }^{k}=\sum_{i=1}^n \lambda_{\phi t_i}^{k}$ \Comment{Party $P_k$}
	    \label{line:ks}
	    \State Set $\bar\lambda_{\phi }^{k}=\lambda_{\phi }^{k}+Lap(\frac{\Delta_f}{\varepsilon})$  \Comment{Party $P_k$}
	    \State Send $\bar\lambda_{\phi }^{k}$ to server \Comment{Party $P_k$}
	    \label{line:ksn}
	    \EndFor
	    \For{each $\lambda_{\phi }^{k*}\in\boldsymbol\lambda_{\phi }^{k*}$}
	    \State Compute $\lambda_{\phi }^{k*}=\sum_{i=1}^n\lambda_{\phi t_i} ^{k*}$ using secure vector multiplication \Comment{Party $P_k,P_*$}
	    \label{line:smc}
	    \State Set $\bar \lambda_{\phi }^{k*}=\lambda_{\phi }^{k*}+Lap(\frac{\Delta_f}{\varepsilon})$ 
	    \Comment{Party $P_k$}
	    \label{line:h}
	    \State Send $\bar\lambda_{\phi }^{k*}$   to server \Comment{Party $P_k$}
	    \label{line:h2}
	    \EndFor
	    \EndFor
		\State Let $\bar{f}_D(\mathbf{w})=\sum\limits_{j=0}^J \sum\limits_{\phi \in \mathbf{\Phi}_j} \left[\bar\lambda^k_{\phi} \phi(\mathbf{w}^k) +\bar\lambda^{k*}_{\phi} \phi(\mathbf{w}^k,\mathbf{w}^*)\right]$, 
		\label{line:fgs} 
		and compute $\bar{\mathbf{w}} = \arg \min\limits_{\mathbf{w}}\bar{f}_D(\mathbf{w})$ \Comment{Server}
		\State Return $\bar{\mathbf{w}}$ \Comment{Server}
		\label{line:return}
	\end{algorithmic}	
	\caption{Functional mechanism in vertically partitioned multiparty learning ($D$,   $f$,    $\varepsilon$)}
	\label{FuncMechs}
\end{algorithm}

\subsubsection{Dissecting objective function.}
In our framework, we only need one round of noise addition. The for-loop in Algorithm \ref{FuncMechs} (Lines 4-15) shows the noise addition and calculation of different subpart/terms of the objective function, all of which together accounts for one single round. In fact, we take advantage that the overall objective function  can be dissected into two parts based on features,
\begin{equation}
\begin{aligned}
    f_D(\mathbf{w}) &= \sum\limits_{k=1}^K g_{D^k}+ \sum\limits_{1\leq k,*\leq K} h_{D^kD^*},
\end{aligned}
\label{eq:polyfuns}
\end{equation}
where $g_{D^k}$ is the single-party sub-function, and  $h_{D^kD^*}$ is the cross-party sub-function.
$g$   only involves data in party $P_k$.
$h$   involves data in party $P_k$ and at least one other party.
Similarly to $f_D(\mathbf{w})$ in Equation \ref{eq:polyfun}, $g$ and $h$ can also be expressed as  polynomials of $w_1,w_2,\cdots,w_d$, $ g_{D^k}
    = \sum_{j=0}^J \sum_{ \phi \in \mathbf{\Phi}_j^{k}} \lambda^k_{\phi} \phi(\mathbf{w}^k)$, $h_{D^kD^*}=  \sum_{j=0}^J \sum_{\phi \in \mathbf{\Phi}_j^{k*}} \lambda^{k*}_{\phi} \phi(\mathbf{w}^k,\mathbf{w}^*)$,
where each $\lambda_{\phi }^{k}$ denotes the single-party coefficient  and  each  $\lambda_{\phi }^{k*}$ denotes the cross-party coefficient.
Note that   $\phi(\mathbf{w}^k,\mathbf{w}^*)$ can be first order parameter ${w}^k$ that needs label from $P_1$, second order cross-party parameter ${w}^k\cdot{w}^l$ or higher order parameter that involves more than two parties.
After dissecting the objective functions, there are two types polynomial coefficients required for the server to obtain the overall objective function, i.e. single-party coefficients $\{\boldsymbol\lambda_{\phi }^{k}\}^K$ and cross-party coefficients $\{\boldsymbol\lambda_{\phi }^{k*}\}^K$.  
Only $D^k$ is required to compute $\lambda_{\phi }^{k}$.
Cross-party computation using $D^k$ and $D^*$ is required to compute $\lambda_{\phi }^{k*}$. 
The server requests $\boldsymbol\lambda_{\phi }^{k}$ from  $P_k$ and $\boldsymbol\lambda_{\phi }^{k*}$ from  $P_k$ and $P_*$. Allocation of single-party and cross-party coefficients is different to each party. It depends on the type of model, the order of parameters and the availability of label.

Take linear regression for an example.
The objective function  of linear regression is as Equation \ref{eq:linfun}.
The single-party sub-function for the label owner, party $P_1$, is $g_{D^1}=\sum\limits_{i=1}^n(y_i)^2  + \left(-2\sum_{i=1}^n y_i{\mathbf{x}^{1}_{i}}^T\right)\mathbf{w}^1 +\sum\limits_{i=1}^n 
\left(\mathbf{x}^{1}_i\right)^2\circ{(\mathbf{w}^1)}^2$, thus $\boldsymbol\lambda_{\phi }^{1}=\{\lambda^1_{\phi_0}, \boldsymbol\lambda_{ \mathbf{w}^1}^{1}, \boldsymbol\lambda_{ {(\mathbf{w}^1)}^2}^{1}\}$.
The single-party sub-function for each  party $P_k (k\neq1)$ is $g_{D^k}= \sum\limits_{i=1}^n 
\left(\mathbf{x}^{1}_i\right)^2$, thus $\boldsymbol\lambda_{\phi }^{k}=\{\boldsymbol\lambda_{ {(\mathbf{w}^k)}^2}^{k}\}$.
The cross-party sub-function is $\sum\limits_{1\leq k,*\leq K}h_{D^kD^*}=\sum\limits_{2=1}^K\left(-2\sum_{i=1}^n y_i{\mathbf{x}^{k}_{i}}^T \mathbf{w}^k\right) + \sum\limits_{1\leq k,l\leq K} \left(\sum\limits_{i=1}^n 
(\mathbf{x}^{k}_i\cdot\mathbf{x}^{l}_i)\circ(\mathbf{w}^k\cdot\mathbf{w}^{l})\right)$, thus $\{\boldsymbol\lambda_{\phi }^{k*}\}^K=\{\boldsymbol\lambda_{  \mathbf{w}^{k}}^{1k},\boldsymbol\lambda_{  \mathbf{w}^k\cdot\mathbf{w}^{l}}^{kl}\}^K$. 
The server sends the formula of coefficients to inform the parties what they need to compute. 
All $(1+d+d^2)$ coefficients in Equation \ref{eq:linfun} are allocated as follow.
For $\lambda_{ \mathbf{\phi}_0}$, it only needs party $P_1$. For first order coefficients $\boldsymbol\lambda_{ \mathbf{\Phi}_1}$,  $\boldsymbol\lambda_{ \mathbf{w}^1}^{1}$ (of size $d^{1}$) from party $P_1$ are single-party coefficients, and $\boldsymbol\lambda_{\mathbf{w}^k}^{1k}$   (of size $d^{k}$) for each party $P_k (k\neq1)$ require communicated information between $P_1$ and  $P_k$ because $P_k$ does not own the label  and need $y_i$ from  $P_1$.  
For  second order coefficients $\boldsymbol\lambda_{  \mathbf{\Phi}_2}  $,  all $\boldsymbol\lambda_{  {(\mathbf{w}^k)}^2}^{k}  $ (of size ${(d^{1})}^2$) are single-party coefficients for each party $P_k$, and $\boldsymbol\lambda_{  \mathbf{w}^k\cdot\mathbf{w}^{l}}^{kl}$ (of size $d^{k}\cdot d^{l}$) require  information from two parties to compute coefficients of $\mathbf{w}^{k}\cdot\mathbf{w}^{l}$ for each pair of    $P_k, P_l$.

\subsubsection{Collecting single-party   and  cross-party coefficients. }
Before each party $P_k$  provides necessary information to the server, the server decides the scale of noise needed for the model to satisfy $\varepsilon$-differential privacy regarding to the whole dataset $D$ based on simply the input space (dimension $d$ and range of $\mathbf{x}_i, y_i$). 
To achieve $\varepsilon$-differential privacy, the functional mechanism adds $Lap(\frac{\Delta_f}{\varepsilon})$ noise to the polynomial coefficients of the objective function. The server calculates the global sensitivity of   objective function $f_D(\mathbf{w})$, and then informs the parties the scale of noise that the parties need to add when sending  results to the server.
\begin{lemma}
The global sensitivity of ${f}_D(\mathbf{w})$ is:
\begin{equation}
\begin{aligned}
\label{eq:delta_fs}
\Delta_f&=2\max\limits_t \sum_{j=1}^{J}\sum_{\phi \in \mathbf{\Phi}_j}||\lambda_{\phi t}||_1.
\end{aligned}
\end{equation}
\end{lemma}
Because the sensitivity only considers the worst case in the input space, the server can calculate the scale of noise needed for global model without getting  raw data from local parties. After   parties receive the coefficients they need to compute and    the amount of noise they need to add onto the results, each party adds Laplace noise $Lap(\frac{\Delta_f}{\varepsilon})$ to both $\boldsymbol\lambda_{\phi }^{k}$ and $\boldsymbol\lambda_{\phi }^{k*}$, and then sends noisy coefficients $\bar{\boldsymbol\lambda}_{\phi }^{k},\bar{\boldsymbol\lambda}_{\phi }^{k*}$ to server.

\subsubsection{Secure vector multiplication.}

When cross-party communication is needed, parties will not share detail data unless through secure multiparty computation methods. 
Because each $\lambda_{\phi }^{k*}=\sum_{i=1}^nv_i^k\cdot v_i^*$ is a scalar product of two vectors $\mathbf{v}^k,\mathbf{v}^*$ from party $P_k$ and $P_*$, the only secure operation required to compute $\boldsymbol\lambda_{\phi }^{k*}$ is   scalar product of two vectors.  We use the secure multiparty computation scheme by \cite{YuanY14} to handle multiparty secure vector multiplication. Each party sends the encrypted vector to  server. The server computes all scalar products without actually knowing  information in the vectors. The participating parties receive the encrypted results back and jointly decrypt the actual results. The server has zero-knowledge on the raw data during the secure vector multiplication process. Then the parties add Laplace noise $Lap(\frac{\Delta_f}{\varepsilon})$ to these cross-party coefficients and send the noisy results to server. 

The total number of operations of secure vector multiplication is $O(dJK)$, and it occurs one and only one round  at the beginning of our approach. For other approaches using secure  aggregation scheme to update gradients \cite{Hardy2017}, the number  increases by at least a magnitude of the number of iterations.

\begin{theorem}
Algorithm \ref{FuncMechs} satisfies $\varepsilon$-differential privacy regarding to $D$.
\label{th:s}
\end{theorem}
\begin{proof}
    Assume $D$ and $D'$ are two neighbouring datasets. 
    Without loss of generality, $D$ and $D'$ differ in  row $t_r$ and $t'_r$. $\Delta_f$ is calculated by Equation \ref{eq:delta_fs}.  We have\\
\resizebox{.95\linewidth}{!}{
\begin{minipage}{\linewidth}
	\begin{equation}
	\begin{aligned}
	&\frac{\Pr\{\bar{f}(\mathbf{w})|D\}}{\Pr\{\bar{f}(\mathbf{w})|D'\}}=\frac{\prod_{j=1}^J\prod_{\phi \in \mathbf{\Phi}_j}\exp\Big(\frac{\varepsilon\big|\big|\sum_{t_i \in D}\lambda_{\phi t_i}-\bar\lambda_{\phi}\big|\big|_1}{\Delta_{{f}}}\Big)}{\prod_{j=1}^J\prod_{\phi \in \mathbf{\Phi}_j}\exp\Big(\frac{\varepsilon\big|\big|\sum_{t_i' \in D'}\lambda_{\phi t_i'}-\bar\lambda_{\phi}\big|\big|_1}{\Delta_{{f}}}\Big)}\\
	&\leq \prod\limits_{j=1}^J\prod\limits_{\phi \in \mathbf{\Phi}_j}\exp\Big(\frac{\varepsilon}{\Delta_{{f}}}\cdot\Big|\Big|\sum_{t_i \in D}\lambda_{\phi t_i}-\sum_{t_i' \in D'}\lambda_{\phi t_i'}\Big|\Big|_1\Big)\\
	\end{aligned}
	\nonumber
	\end{equation}
\end{minipage}
}
\resizebox{.9\linewidth}{!}{
\begin{minipage}{\linewidth}
	\begin{equation}
	\begin{aligned}
	&= \prod\limits_{j=1}^J\prod\limits_{\phi \in \mathbf{\Phi}_j}\exp\Big(\frac{\varepsilon}{\Delta_{{f}}}\cdot\big|\big|\lambda_{\phi t_r}-\lambda_{\phi t_r'}\big|\big|_1\Big)\\
	&= \exp\Big(\frac{\varepsilon}{\Delta_{{f}}}\cdot\sum\limits_{j=1}^J\sum\limits_{\phi \in \mathbf{\Phi}_j}\big|\big|\lambda_{\phi t_r}-\lambda_{\phi t_r'}\big|\big|_1\Big) \\
	&\leq \exp\Big(\frac{\varepsilon}{\Delta_{{f}}}\cdot 2\max\limits_t \sum_{j=1}^{J}\sum_{\phi \in \mathbf{\Phi}_j}||\lambda_{\phi t}||_1 \Big) =\exp(\varepsilon).
	\end{aligned}
	\nonumber
	\end{equation}
\end{minipage}
}
\end{proof}

\begin{table*}
\small
\centering	
\caption{Mean square error of linear regression on US and Brazil datasets under different privacy budgets $\varepsilon$ ($\delta=\tfrac{1}{n}$ for DPFW)}
\label{tbl:lin}
\begin{tabular}{|c|c|c|c|c|c|}
\hline
Data                    & $\varepsilon$ & non-private                     & DPFW-C       & DPFW-2       & FM                      \\ \hline
\multirow{3}{*}{US}     & 0.1     & \multirow{3}{*}{0.0044$\pm$0.0132} & 0.0912$\pm$0.0007 & 0.1286$\pm$0.0154 & \textbf{0.0101$\pm$0.0305} \\ \cline{2-2} \cline{4-6} 
                        & 1       &                                 & 0.0905$\pm$0.0008 & 0.1334$\pm$0.0235 & \textbf{0.0087$\pm$0.0261} \\ \cline{2-2} \cline{4-6} 
                        & 10      &                                 & 0.0903$\pm$0.0008 & 0.1329$\pm$0.0204 & \textbf{0.0086$\pm$0.0259} \\ \hline
\multirow{3}{*}{Brazil} & 0.1     & \multirow{3}{*}{0.0044$\pm$0.0132} & 0.0470$\pm$0.0007 & 0.1502$\pm$0.0921 & \textbf{0.0070$\pm$0.0230} \\ \cline{2-2} \cline{4-6} 
                        & 1       &                                 & 0.0454$\pm$0.0007 & 0.1989$\pm$0.1390 & \textbf{0.0045$\pm$0.0136} \\ \cline{2-2} \cline{4-6} 
                        & 10      &                                 & 0.0453$\pm$0.0006 & 0.2013$\pm$0.1469 & \textbf{0.0044$\pm$0.0132} \\ \hline
\end{tabular}
\end{table*}

In Theorem \ref{th:s}, Algorithm \ref{FuncMechs} satisfies $\varepsilon$-differential privacy regarding $D$, which is the same as the centralized scenario. In comparison to the centralized functional mechanism, the proposed framework adds the same amount of noise to achieve $\varepsilon$-differential privacy and uses secure vector multiplication to achieve the same utility. In comparison to the methods that add noise onto gradients for each iteration, our framework only needs one round of noise addition and one round of secure multiparty computation.
\begin{claim}
Algorithm \ref{FuncMechs} achieves the same utility under the multiparty setting in comparison to the centralized setting. The utility of Algorithm \ref{FuncMechs}  does not change along with the number of participating parties $K$.
\end{claim}

\subsubsection{Differential privacy for local parties.}
Each party  $P_k$ cares about all the coefficients that involve $D^k$, i.e.  $\boldsymbol\lambda_{\phi }^{k}$ in $g$ and $ \boldsymbol\lambda_{\phi }^{k*}$ in $h$. 
\begin{lemma}
The sensitivity of ${f}_D(\mathbf{w})$  regarding to $D^{k}$ is 
\begin{equation}
\label{eq:delta_fAs}
\Delta_f^{k}=2\max\limits_{t} \sum_{j=1}^{J}\sum_{\phi \in \mathbf{\Phi}_j^{(k)}}||\lambda^{(k)}_{\phi t}||_1,
\end{equation}
where $\lambda^{(k)}_{\phi}$ indicates either $\lambda_{\phi }^{k}$ or $\lambda_{\phi }^{k*}$.
\end{lemma}
The availability of label makes significant difference in $\Delta_f^{k}$. Because  party $P_1$    owns the label and the label  is granted access to other parties, there are more cross-party coefficients for  $P_1$ than other parties, which increases $\Delta_f^{1}$.

\begin{theorem}
Algorithm \ref{FuncMechs} satisfies $\varepsilon^{(k)}$-differential privacy regarding to $D^{k}$, where $\varepsilon^{(k)}=\frac{\Delta_f^{k}}{\Delta_f}\varepsilon$.
\label{th:As}
\end{theorem}
\begin{proof}
	Assume that $D^{k}$ and ${D^{k}}'$ are two neighbouring datasets. 
	Without loss of generality, $D^{k}$ and ${D^{k}}'$ differ in  row $t_r$ and ${t'_r}$. $\Delta_f^{k}$ is calculated by Equation \ref{eq:delta_fAs}. 
	We have\\
\resizebox{.9\linewidth}{!}{
\begin{minipage}{\linewidth}
	\begin{equation}
	\begin{aligned}
	&\frac{\Pr\{\bar{f}(\mathbf{w})|D^{k}\}}{\Pr\{\bar{f}(\mathbf{w})|{D^{k}}'\}} =\frac{\prod\limits_{j=1}^J\prod\limits_{\phi \in \mathbf{\Phi}^{(k)}_j}\exp\Big(\frac{\varepsilon\big|\big|\sum_{t_i \in D^{k}}\lambda^{(k)}_{\phi t_i}-\bar\lambda^{(k)}_{\phi}\big|\big|_1}{\Delta_{{f}}}\Big)}{\prod\limits_{j=1}^J\prod\limits_{\phi \in \mathbf{\Phi}^{(k)}_j}\exp\Big(\frac{\varepsilon\big|\big|\sum_{{t'_i} \in {D^{k}}'}\lambda^{(k)}_{\phi {t'_i}}-\bar\lambda^{(k)}_{\phi}\big|\big|_1}{\Delta_{{f}}}\Big)}\\
	&\leq \exp\Big(\frac{\varepsilon}{\Delta_{{f}}}\cdot 2\max\limits_{t} \sum_{j=1}^{J}\sum_{\phi \in \mathbf{\Phi}^{(k)}_j}||\lambda^{(k)}_{\phi t}||_1 \Big) =\exp(\frac{\Delta_f^{k}}{\Delta_f}\varepsilon).
	\end{aligned}
	\nonumber
	\end{equation}
\end{minipage}
}
\end{proof}

To build an $\varepsilon$-differential privacy global model, the local party $P_k$ can achieve 
$\varepsilon^{(k)}$-differential privacy, where $\varepsilon^{(k)}=\frac{\Delta_f^{k}}{\Delta_f}\varepsilon < \varepsilon$, which means stronger privacy guarantee.

Again, take linear regression for an example. To achieve $\varepsilon$-differential privacy regarding to $D$, functional mechanism adds $Lap(\frac{\Delta_f }{\varepsilon})$ noise to polynomial coefficients. The global sensitivity of ${f}_D(\mathbf{w})$ regarding to $D$ is $\Delta_f=2(1+2d+d^2)$.
Party $P_1$ cares about $\lambda^1_{{\phi_0 } },\boldsymbol\lambda_{ {\mathbf{w}^1} }^{1}, \boldsymbol\lambda_{ {\mathbf{w}^k}}^{1k}, \boldsymbol\lambda_{ {(\mathbf{w}^1)}^2}^{1}$ and  $\boldsymbol\lambda_{{\mathbf{w}^1\cdot\mathbf{w}^{k}}}^{1k} (k\neq 1)$.
So the sensitivity of ${f}_D(\mathbf{w})$  regarding to $D^{1}$ is $\Delta_f^{1}=2(1+2d^{1}+2{(d-d^{1})}+{(d^{1})}^2+{{d^{1}}{(d-d^{1})}})= 2(1+2d+d^{1}d)$.
For party $P_1$, Algorithm \ref{FuncMechs} achieves $(\frac{\Delta_f^{1}}{\Delta_f}\varepsilon)$-differential privacy regarding to $D^{1}$.
Party $P_k (k\neq 1)$  cares about $\boldsymbol\lambda_{ {\mathbf{w}^k}}^{1k}, \boldsymbol\lambda_{{(\mathbf{w}^k)}^2}^{k}$ and  $\boldsymbol\lambda_{{\mathbf{w}^k\cdot\mathbf{w}^{l}}}^{kl}$.
So the sensitivity of ${f}_D(\mathbf{w})$  regarding to $D^k (k\neq 1)$ is  $\Delta_f^{k}=2(2d^{k}+{({d^{k})}^2}+{{d^{k}}{(d-d^{k})}})= 2(2d^{k}+d^{k}d)$.
For party $P_k (k\neq 1)$, Algorithm \ref{FuncMechs} achieves $(\frac{\Delta_f^{k}}{\Delta_f}\varepsilon)$-differential privacy regarding to $D^{k}$.
The interesting observation is that: when party $P_1$ shares label with other parties,  $P_1$ has $4(d-d^{1})$ more sensitivity. $\Delta_f^{k}$ for other parties does not change as the label $|y|\leq 1$. Thus, cross-party communication only costs  $P_1$ extra sensitivity. By sharing the label, party $P_1$ achieves relatively weaker privacy in comparison to  other parties.

\begin{table*}
\small
\caption{Mean square error of linear regression on synthetic datasets under different sparsity $s$ ($\delta=\tfrac{1}{n}$ for DPFW, $\varepsilon=1$)}
\label{tbl:lins}
	\centering	
	\begin{tabular}{|c|c|c|c|c|c|c|}
\hline
$s$ & non-private    & DPFW-C       & DPFW-2       & DPFW-4       & DPFW-8       & FM                      \\ \hline
0.1      & 0.0033$\pm$0.0101 & 0.5487$\pm$0.1907 & 1.2980$\pm$0.4537 & 2.2719$\pm$0.7423 & 4.6538$\pm$2.0663 & \textbf{0.0035$\pm$0.0107} \\ \hline
0.5      & 0.0052$\pm$0.0157 & 40.392$\pm$5.132  & 85.264$\pm$10.560 & 164.95$\pm$19.60  & 338.16$\pm$31.41  & \textbf{0.0058$\pm$0.0174} \\ \hline
1.0        & 0.0050$\pm$0.0154 & 271.09$\pm$27.75  & 638.66$\pm$59.70  & 1206.4$\pm$66.2   & 2441.0$\pm$174.2  & \textbf{0.0056$\pm$0.0172} \\ \hline
\end{tabular}
\end{table*}

\begin{table*}
\small
\caption{Classification accuracy of logistic regression on Adult and Dutch datasets under different privacy budget  $\varepsilon$}
\label{tbl:log}
	\centering	
\begin{tabular}{|c|c|c|c|c|c|c|}
\hline
\multirow{2}{*}{$\varepsilon$} & \multicolumn{3}{c|}{Adult}                                                 & \multicolumn{3}{c|}{Dutch}                                                 \\ \cline{2-7} 
                         & non-private                     & DPSGD          & FM                      & non-private                     & DPSGD          & FM                      \\ \hline
0.1                      & \multirow{3}{*}{0.8368$\pm$0.0029} & 0.6000$\pm$0.0738 & \textbf{0.6412$\pm$0.1463} & \multirow{3}{*}{0.8303$\pm$0.0040} & 0.5060$\pm$0.0684 & \textbf{0.5783$\pm$0.0572} \\ \cline{1-1} \cline{3-4} \cline{6-7} 
1                        &                                 & 0.6956$\pm$0.0229 & \textbf{0.7315$\pm$0.0379} &                                 & 0.6867$\pm$0.0373 & \textbf{0.7166$\pm$0.0489} \\ \cline{1-1} \cline{3-4} \cline{6-7} 
10                       &                                 & 0.8023$\pm$0.0071 & \textbf{0.8132$\pm$0.0231} &                                 & 0.8003$\pm$0.0182 & \textbf{0.8105$\pm$0.0086} \\ \hline
\end{tabular}
\end{table*}

\subsection{Application to Logistic Regression}
For logistic regression, to achieve $\varepsilon$-differential privacy, the functional mechanism adds $Lap(\frac{\Delta_f}{\varepsilon})$ noise to the polynomial coefficients in Equation \ref{eq:polylog}.
More specifically, $\boldsymbol\lambda_{ \mathbf{\Phi}_1}$ contains $\boldsymbol\lambda_{  {\mathbf{w}^1}}^{1}= \sum\limits_{i=1}^n 
\big(\tfrac{f_1^{(1)}(0)}{1!}-y_i\big)\mathbf{x}^{1}_i$ from   $P_1$ and $\boldsymbol\lambda_{ {\mathbf{w}^k}}^{1k}=\sum\limits_{i=1}^n
\big(\tfrac{f_1^{(1)}(0)}{1!}-y_i\big)\mathbf{x}^{k}_i$ where $P_k (k\neq 1)$ does not own the label  and need $(\frac{f_1^{(1)}(0)}{1!}-y_i)$ from   $P_1$.  
$\boldsymbol\lambda_{  \mathbf{\Phi}_2}  $ contains $\boldsymbol\lambda_{ {(\mathbf{w}^k)}^2}^{k} = \sum\limits_{i=1}^n 
\tfrac{f_1^{(2)}(0)}{2!}\left(\mathbf{x}^{k}_i\right)^2 $ from   $P_k$ and $\boldsymbol\lambda_{\mathbf{w}^k\cdot\mathbf{w}^{l}}^{kl}= \sum\limits_{i=1}^n 
\tfrac{f_1^{(2)}(0)}{2!}\mathbf{x}^{k}_i\cdot\mathbf{x}^{l}_i$ from $P_k,P_l$. 

To build  the global model, the derived $\bar{\mathbf{w}}$ satisfies $\varepsilon$-differential privacy regarding to $D$ by applying Algorithm \ref{FuncMechs}.
For   $P_1$,   $\Delta_f^{1}=d+ \frac{{d^{1}}({2d}-{d^{1}})}{4}$. 
For   $P_k (k\neq 1)$,   $\Delta_f^{k}=d^{k}+\frac{d^k{(2d-d^{k})}}{4}$. 
Algorithm \ref{FuncMechs} achieves $(\frac{\Delta_f^{1}}{\Delta_f}\varepsilon)$-differential privacy regarding to $D^{1}$ and  $(\frac{\Delta_f^{k}}{\Delta_f}\varepsilon)$-differential privacy regarding to $D^{k}$.
When   $P_1$ shares label with other parties,   $P_1$ has  $(d-d^{1})$ more sensitivity, and $\Delta_f^{k}$ for other parties does not change as the label $y\in\{0,1\}$  does not change $|(\frac{f_1^{(1)}(0)}{1!}-y)|=\frac{1}{2}$.

\subsection{Extension to the Bottom-up Case}
So far, we have discussed the framework from the top-down case, where the server selects the privacy budget to achieve on the whole dataset and informs the parties the scale of noise based on the global sensitivity of objective function. We can also achieve differential privacy from the bottom-up case, where each party selects the level of differential privacy $\varepsilon^{(k)}$ they want to achieve for their sub-dataset and the server adjusts    $\varepsilon$ accordingly. In practice, the choice between top-down and bottom-up approaches depends on the agreement between server and parties.

In the bottom-up case, each party $P_k$ splits their privacy budget $\varepsilon^{(k)}$ onto sending ${\boldsymbol\lambda}_{\phi }^{k}$ and ${\boldsymbol\lambda}_{\phi }^{k*}$ in a differentially private manner, i.e.   $\varepsilon^{(k)}=\varepsilon^k+\sum_{l=1}^K\varepsilon^{kl}$, where  $\varepsilon^k$ is the privacy budget for   $\boldsymbol\lambda_{\phi }^{k}$, and $\varepsilon^{kl}$ is the privacy budget for    $ \boldsymbol\lambda_{\phi }^{kl} \subset \boldsymbol\lambda_{\phi }^{k*}$. $P_k$ and $P_l$ jointly decide $\varepsilon^{kl}$.  The sensitivity of $\boldsymbol\lambda_{\phi }^{k}$ is $\Delta_g^{k}=2\max\limits_{t} \sum_{j=1}^{J}\sum_{\phi \in \mathbf{\Phi}_j^{k}}||\lambda^{k}_{\phi t}||_1$. The sensitivity of $ \boldsymbol\lambda_{\phi }^{kl}$ is $\Delta_h^{kl}=2\max\limits_{t} \sum_{j=1}^{J}\sum_{\phi \in \mathbf{\Phi}_j^{kl}}||\lambda^{kl}_{\phi t}||_1$.
\begin{corollary}
The global model achieves $\varepsilon$-differential privacy regarding to $D$ in the bottom-up case, where
\begin{equation}
\begin{aligned}
    \varepsilon &= \sum\limits_{k=1}^K \frac{\Delta_f}{\Delta_g^{k}}\varepsilon^k+ \sum\limits_{1\leq k,l\leq K} \frac{\Delta_f}{\Delta_h^{kl}}\varepsilon^{kl}.
\end{aligned}
\nonumber
\end{equation}
\end{corollary}

\subsection{Discussion}
The objective of our work is to preserve differential privacy for regression models trained on vertically partitioned data. We have theoretical proofs (Theorems \ref{th:s} and \ref{th:As} ) that our algorithm guarantees to satisfy differential privacy. We protect differential privacy of the whole training data such that attackers cannot derive the presence or absence of any single record (with all feature values and label) in the training data from the released jointly-learnt regression model (as shown in Theorem \ref{th:s}). As training data is vertically split into $K$ parties, we further show in Theorem \ref{th:As} each party $k$ also achieves differential privacy against attackers regarding to its own data $D^k$. 

Our contribution is that we add less noise than state-of-the-art approaches to achieve the same level of differential privacy, e.g., our approach reduces noise addition by a magnitude of the number of total iterations compared with gradient perturbation approaches. Furthermore, the secure vector computation also protects the disclosure of raw data between the server and the participating parties. We use the standard secure vector multiplication under the semi-honest model in our framework. The number of inner products in our algorithm is bounded by $d^2$.

\section{Experiments}

We evaluate our proposed framework of achieving differential privacy in vertically partitioned multiparty learning based on functional mechanism (FM) for linear regression and logistic regression.
\subsection{Experiment Setup}

\subsubsection{Dataset.} 
For linear regression, we evaluate on US and Brazil \cite{ipums} datasets.  US has 370,000 records and 14 features and  Brazil has 190,000 and 14 features. We also evaluate on three synthetic datasets that are sparse and high dimensional. All three synthetic datasets have 80,000 records and 800 features and their sparsity values are  $s=0.1,0.5,1.0$, respectively. The sparsity here is both the ratio of nonzero entries in datasets and the ratio of non-zero ground-truth parameters.   For logistic regression, we evaluate on Adult \cite{Dua:2017} and Dutch \cite{dutch} datasets.  Adult has 45,222 samples and 41 features and Dutch has 60,420 records and 36 features. We split each dataset into 80\% training data and 20\% testing data. We replicate experiment for 10 times and report  mean and standard deviation.

\subsubsection{Baseline.} 
For linear regression, we compare with the non-private linear regression and DPFW \cite{LouC18}.  DPFW achieves $(\varepsilon,\delta)$-differential privacy  and has two versions, DPFW-C in the centralized setting and DPFW-K in the multiparty setting. We specify the number of parties K=2,4,8 in our comparison.
For logistic regression, we compare with the non-private logistic regression and DPSGD \cite{SongCS13}.  DPSGD adds Laplace noise onto gradients for each iteration. DPSGD does not apply to multiparty setting. 

We evaluate   utility of linear regression by mean square error (MSE) and   utility of  logistic regression by   accuracy. 

\subsection{Linear Regression}
For linear regression, we first evaluate our method on two real-world datasets. Table \ref{tbl:lin} shows the results on US and Brazil datasets under different values (0.1, 1, 10) of privacy budget $\varepsilon$. We set $\delta=\tfrac{1}{n}$ for DPFW. Our FM method satisfies $(\varepsilon,0)$-differential privacy whereas DPFW satisfies $(\varepsilon,\delta)$-differential privacy. So our FM method is more restricted in terms of privacy protection. However, our FM method still significantly outperforms DPFW  with much smaller MSE under the settings of all three $\varepsilon$ values for both datasets, as shown in Table \ref{tbl:lin}. In fact, the utility of our method is very close to the non-private linear regression even when $\varepsilon$ is small. For example, our FM achieves the MSE of 0.0070 when $\varepsilon=0.1$ for Brazil data, which is very close to 0.0044 from non-private linear regression.   

We then evaluate our method on high dimensional synthetic datasets ($d=800$). 
Table \ref{tbl:lins} shows the results on synthetic datasets under different sparsity $s$. We set $\varepsilon=1$ for FM and  $\varepsilon=1,\delta=\tfrac{1}{n}$ for DPFW. DPFW is designed to work for high dimensional and sparse data. As shown in  \ref{tbl:lins}, DPFW works well with satisfactory MSE values when $s=0.1$ but has very poor utility with large MSE when $s=0.5,1.0$. On the contrary, our FM method works consistently well across all three datasets as the FM mechanism does not depend on data sparsity. We emphasize even with $s=0.1$, our FM method incurs much smaller MSE (2 or 3 orders of magnitude less) than DPFW. Moreover, our method preserves strict $(\varepsilon,0)$-differential privacy while DPFW preserves $(\varepsilon,\frac{1}{n})$-differential privacy.  Because our method achieves the same utility in the decentralized setting as in the centralized setting, MSE does not change  along with the number of participating parties $K$. On contrast, DPFW incurs more  utility loss as $K$ increases.

\subsection{Logistic Regression}
For logistic regression, we evaluate our method on two real-world datasets. Table \ref{tbl:log} shows the results on Adult and Dutch datasets under different privacy budget $\varepsilon$.  DPSGD adds Laplace noise onto gradients for each iteration, so the total amount of noise added into the model increases proportionally with the number of iterations. On the contrary, our FM  only adds noise to the objective function and only adds once. As shown in Table \ref{tbl:log},  the utility of our method is much better than DPSGD. Moreover, DPSGD cannot apply to the multiparty setting while our method is applicable and independent of $K$. We also would like to point out that, compared to linear regression, the utility of FM is worse as privacy budget decreases. This is because the order-2 Taylor expansion approximation is biased to the original objective function.

\section{Conclusions and Future Work}
We proposed a new framework for differential privacy preserving multiparty learning in the vertically partitioned setting based on the functional mechanism.
In the framework,   the server     dissects the objective function into single-party and cross-party sub-functions and rewrite them in the polynomial form. For the coefficients in the polynomial form related to one single party, they can be calculated by each party. For those coefficients related to two or multiple parties, we apply secure vector multiplication.
To achieve differential privacy, the parties add noise to the coefficients according to global sensitivity  and  send noisy coefficients back to server.
The server then solves the perturbed objective function and releases the private model.  Our method needs only one round of noise addition and secure aggregation. The released model in our framework achieves the same utility as applying the functional mechanism in the centralized setting.
Our evaluation on real-world and synthetic datasets for linear and logistic regressions shows the effectiveness of our proposed method.

In our framework, we proposed the use of the BGN doubly homomorphic encryption algorithm for secure inner product calculation. Secure calculation is the bottleneck of our framework as the noise addition of achieving differential privacy via functional mechanism is insignificant in terms of computation and communication cost. In our experiment, we mainly evaluated accuracy of regression models on varying numbers of parties $K$ (the number of features owned by a party when evenly distributed is $d/K$). Our theoretical analysis also showed that our algorithm can achieve the same accuracy as the centralized private model regardless of the number of parties. In our future work, we will evaluate performance due to the change of the number of features and study the performance overhead of BGN.

\section*{Acknowledgments}
This work was supported in part by NSF 1502273, 1920920, 1937010.


\begin{thebibliography}{}

\bibitem[\protect\citeauthoryear{Abadi \bgroup et al\mbox.\egroup
  }{2016}]{AbadiCGMMT016}
Abadi, M.; Chu, A.; Goodfellow, I.~J.; McMahan, H.~B.; Mironov, I.; Talwar, K.;
  and Zhang, L.
\newblock 2016.
\newblock Deep learning with differential privacy.
\newblock In {\em Proceedings of the 2016 {ACM} {SIGSAC} Conference on Computer
  and Communications Security, Vienna, Austria, October 24-28, 2016},
  308--318.

\bibitem[\protect\citeauthoryear{Boneh, Goh, and Nissim}{2005}]{BonehGN05}
Boneh, D.; Goh, E.; and Nissim, K.
\newblock 2005.
\newblock Evaluating 2-dnf formulas on ciphertexts.
\newblock In {\em Theory of Cryptography, Second Theory of Cryptography
  Conference, {TCC} 2005, Cambridge, MA, USA, February 10-12, 2005,
  Proceedings},  325--341.

\bibitem[\protect\citeauthoryear{Chaudhuri, Monteleoni, and
  Sarwate}{2011}]{Chaudhuri2011}
Chaudhuri, K.; Monteleoni, C.; and Sarwate, A.~D.
\newblock 2011.
\newblock Differentially private empirical risk minimization.
\newblock {\em J. Mach. Learn. Res.} 12:1069--1109.

\bibitem[\protect\citeauthoryear{Dheeru and Taniskidou}{2017}]{Dua:2017}
Dheeru, D., and Taniskidou, E.~K.
\newblock 2017.
\newblock {UCI} machine learning repository.

\bibitem[\protect\citeauthoryear{Dwork \bgroup et al\mbox.\egroup
  }{2006}]{Dwork2006}
Dwork, C.; McSherry, F.; Nissim, K.; and Smith, A.~D.
\newblock 2006.
\newblock Calibrating noise to sensitivity in private data analysis.
\newblock In {\em Theory of Cryptography, Third},  265--284.

\bibitem[\protect\citeauthoryear{Dwork}{2011}]{Dwork:2011}
Dwork, C.
\newblock 2011.
\newblock A firm foundation for private data analysis.
\newblock {\em Commun. ACM} 54(1):86--95.

\bibitem[\protect\citeauthoryear{Hardy \bgroup et al\mbox.\egroup
  }{2017}]{Hardy2017}
Hardy, S.; Henecka, W.; Ivey{-}Law, H.; Nock, R.; Patrini, G.; Smith, G.; and
  Thorne, B.
\newblock 2017.
\newblock Private federated learning on vertically partitioned data via entity
  resolution and additively homomorphic encryption.
\newblock {\em CoRR} abs/1711.10677.

\bibitem[\protect\citeauthoryear{Heinze{-}Deml, McWilliams, and
  Meinshausen}{2017}]{Heinze-DemlMM17}
Heinze{-}Deml, C.; McWilliams, B.; and Meinshausen, N.
\newblock 2017.
\newblock Preserving differential privacy between features in distributed
  estimation.
\newblock {\em CoRR} abs/1703.00403.

\bibitem[\protect\citeauthoryear{IPUMS}{2009}]{ipums}
IPUMS.
\newblock 2009.
\newblock Integrated public use microdata series, international: Version 7.2.

\bibitem[\protect\citeauthoryear{Lou and Cheung}{2018}]{LouC18}
Lou, J., and Cheung, Y.
\newblock 2018.
\newblock Uplink communication efficient differentially private sparse
  optimization with feature-wise distributed data.
\newblock In {\em Proceedings of the Thirty-Second {AAAI} Conference on
  Artificial Intelligence, (AAAI-18), the 30th innovative Applications of
  Artificial Intelligence (IAAI-18), and the 8th {AAAI} Symposium on
  Educational Advances in Artificial Intelligence (EAAI-18), New Orleans,
  Louisiana, USA, February 2-7, 2018},  125--133.

\bibitem[\protect\citeauthoryear{McMahan \bgroup et al\mbox.\egroup
  }{2017}]{McMahanMRHA17}
McMahan, B.; Moore, E.; Ramage, D.; Hampson, S.; and y~Arcas, B.~A.
\newblock 2017.
\newblock Communication-efficient learning of deep networks from decentralized
  data.
\newblock In {\em Proceedings of the 20th International Conference on
  Artificial Intelligence and Statistics, {AISTATS} 2017, 20-22 April 2017,
  Fort Lauderdale, FL, {USA}},  1273--1282.

\bibitem[\protect\citeauthoryear{McSherry and Talwar}{2007}]{4389483}
McSherry, F., and Talwar, K.
\newblock 2007.
\newblock Mechanism design via differential privacy.
\newblock In {\em FOCS}.

\bibitem[\protect\citeauthoryear{Nissim, Raskhodnikova, and
  Smith}{2007}]{Nissim:2007}
Nissim, K.; Raskhodnikova, S.; and Smith, A.
\newblock 2007.
\newblock Smooth sensitivity and sampling in private data analysis.
\newblock In {\em SOTC}.

\bibitem[\protect\citeauthoryear{Phan \bgroup et al\mbox.\egroup
  }{2016}]{DBLP:conf/aaai/Phan0WD16}
Phan, N.; Wang, Y.; Wu, X.; and Dou, D.
\newblock 2016.
\newblock Differential privacy preservation for deep auto-encoders: an
  application of human behavior prediction.
\newblock In {\em Proceedings of the Thirtieth {AAAI} Conference on Artificial
  Intelligence, February 12-17, 2016, Phoenix, Arizona, {USA.}},  1309--1316.

\bibitem[\protect\citeauthoryear{Rudin}{1953}]{rudin1953principles}
Rudin, W.
\newblock 1953.
\newblock {\em Principles of mathematical analysis}.
\newblock International series in pure and applied mathematics. McGraw-Hill.

\bibitem[\protect\citeauthoryear{Shokri and Shmatikov}{2015}]{Shokri2015}
Shokri, R., and Shmatikov, V.
\newblock 2015.
\newblock Privacy-preserving deep learning.
\newblock In {\em Proceedings of the 22Nd ACM SIGSAC Conference on Computer and
  Communications Security}, CCS '15,  1310--1321.

\bibitem[\protect\citeauthoryear{Song, Chaudhuri, and Sarwate}{2013}]{SongCS13}
Song, S.; Chaudhuri, K.; and Sarwate, A.~D.
\newblock 2013.
\newblock Stochastic gradient descent with differentially private updates.
\newblock In {\em GlobalSIP},  245--248.

\bibitem[\protect\citeauthoryear{Yang \bgroup et al\mbox.\egroup
  }{2019}]{YangLCT19}
Yang, Q.; Liu, Y.; Chen, T.; and Tong, Y.
\newblock 2019.
\newblock Federated machine learning: Concept and applications.
\newblock {\em {ACM} {TIST}} 10(2):12:1--12:19.

\bibitem[\protect\citeauthoryear{Yao \bgroup et al\mbox.\egroup
  }{2019}]{YaoGKTCD019}
Yao, Q.; Guo, X.; Kwok, J.~T.; Tu, W.; Chen, Y.; Dai, W.; and Yang, Q.
\newblock 2019.
\newblock Privacy-preserving stacking with application to cross-organizational
  diabetes prediction.
\newblock In {\em Proceedings of the Twenty-Eighth International Joint
  Conference on Artificial Intelligence, {IJCAI} 2019, Macao, China, August
  10-16, 2019},  4114--4120.

\bibitem[\protect\citeauthoryear{Yuan and Yu}{2014}]{YuanY14}
Yuan, J., and Yu, S.
\newblock 2014.
\newblock Privacy preserving back-propagation neural network learning made
  practical with cloud computing.
\newblock {\em {IEEE} Trans. Parallel Distrib. Syst.} 25(1):212--221.

\bibitem[\protect\citeauthoryear{Zhang \bgroup et al\mbox.\egroup
  }{2012}]{zhang2012functional}
Zhang, J.; Zhang, Z.; Xiao, X.; Yang, Y.; and Winslett, M.
\newblock 2012.
\newblock Functional mechanism: regression analysis under differential privacy.
\newblock {\em PVLDB} 5(11):1364--1375.

\bibitem[\protect\citeauthoryear{Zliobaite, Kamiran, and Calders}{2011}]{dutch}
Zliobaite, I.; Kamiran, F.; and Calders, T.
\newblock 2011.
\newblock Handling conditional discrimination.
\newblock In {\em ICDM}.

\end{thebibliography}

\end{document}